\documentclass{article}

\usepackage{microtype}
\usepackage{graphicx}
\usepackage{subfigure}
\usepackage{booktabs} 
\usepackage[bb=boondox]{mathalfa}
\usepackage{amsmath}

\usepackage{hyperref}
\usepackage{amsthm}
\usepackage{wrapfig}

\usepackage[accepted]{icml2020}

\icmltitlerunning{``Other-Play'' for Zero-Shot Coordination}

\begin{document}

\twocolumn[
\icmltitle{``Other-Play'' for Zero-Shot Coordination}




\icmlsetsymbol{equal}{*}
\newtheorem{theorem}{Theorem}
\newtheorem{corollary}{Corollary}
\newtheorem{prop}{Proposition}
\newtheorem{lemma}{Lemma}

\begin{icmlauthorlist}
\icmlauthor{Hengyuan Hu}{equal,fair}
\icmlauthor{Adam Lerer}{fair}
\icmlauthor{Alex Peysakhovich}{fair}
\icmlauthor{Jakob Foerster}{equal,fair}
\end{icmlauthorlist}

\icmlaffiliation{fair}{Facebook AI Research, USA}

\icmlcorrespondingauthor{Hengyuan Hu}{hengyuan@fb.com}
\icmlcorrespondingauthor{Jakob Foerster}{jnf@fb.com}

\icmlkeywords{Machine Learning, ICML}

\vskip 0.3in
]


%
\printAffiliationsAndNotice{\icmlEqualContribution} 

\begin{abstract}
We consider the problem of zero-shot coordination - constructing AI agents that can coordinate with novel partners they have not seen before (e.g. humans). Standard Multi-Agent Reinforcement Learning (MARL) methods typically focus on the self-play (SP) setting where agents construct strategies by playing the game with themselves repeatedly. Unfortunately, applying SP naively to the zero-shot coordination problem can produce agents that establish highly specialized conventions that do not carry over to novel partners they have not been trained with. We introduce a novel learning algorithm called \emph{other-play} (OP), that enhances self-play by looking for more robust strategies, exploiting the presence of known symmetries in the underlying problem. We characterize OP theoretically as well as experimentally. We study the cooperative card game Hanabi and show that OP agents achieve higher scores when paired with independently trained agents. In preliminary results we also show that our OP agents obtains higher average scores when paired with human players, compared to state-of-the-art SP agents.
\end{abstract}
\section{Introduction}
\label{sec:intro}
A central challenge for AI is constructing agents that can coordinate and cooperate with partners they have not seen before \citep{kleiman2016coordinate,lerer2017maintaining,carroll2019utility,shum2019theory}. This is particularly important in applications such as cooperative game playing, communication, or autonomous driving \cite{foerster2016learning,lazaridou2016multi,sukhbaatar2016learning,resnick2018vehicle}. In this paper we consider the question of zero-shot coordination where agents are placed into a cooperative situation with a novel partner and must quickly coordinate if they wish to earn high payoffs.

Our setting is a partially observed cooperative Markov game (MG) which is commonly known among both agents. The agents are able to construct strategies separately in the training phase but cannot coordinate on the strategies that they construct. They must then play these strategies when paired together one time. We refer to this as zero-shot coordination.

A popular way of constructing strategies for MG with unknown opponents is self-play (or ``self-training"), \cite{tesauro1994td}. Here the agent controls both players during training and iteratively improves both players' strategies. The agent then uses this strategy at test time. If it converges, self-play finds a Nash equilibrium of the game and yields superhuman AI agents in two-player zero-sum games such as Chess, Go and Poker \citep{campbell2002deep,silver2017mastering,brown2018superhuman}. However, in complex environments self-play agents typically construct `inhuman' strategies \citep{carroll2019utility}. This may be a benefit for zero-sum games, but is less useful when it is important to coordinate with, not trick, one's partner.

Our main contribution is ``other-play'' (OP), an algorithm for constructing good strategies for the zero-shot coordination setting. We assume that with every MG we are provided with a set of symmetries, i.e. arbitrary relabelings of the state/action space that leave trajectories unchanged up to relabeling. One source of miscoordination in zero-shot settings is that agents have no good way to break the symmetries (e.g. should we drive on the left or the right?). In most MDPs, there are classes of strategies that require more or less coordinated symmetry breaking. OP's goal is to find a strategy that is \textit{maximally robust} to partners breaking symmetries in different ways while still playing in the same class. OP works as follows: it uses RL to maximize reward when matched with agents playing the same policy under a random relabeling of states and actions under the known symmetries.

\begin{figure}[h]
\vskip 0.2in
\begin{center}
\centerline{\includegraphics[width=\columnwidth]{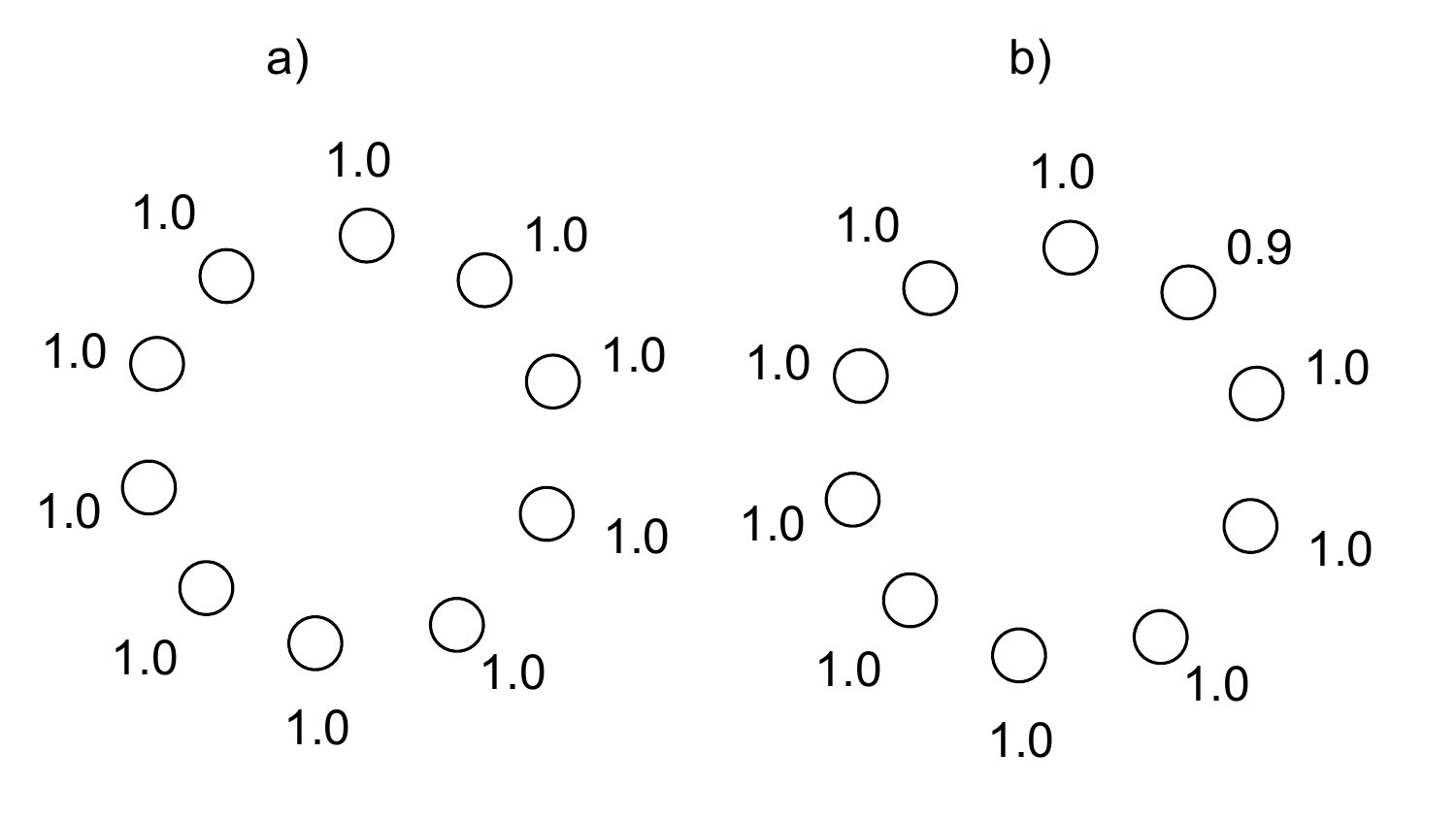}}
\caption{The \emph{lever coordination game} illustrates the counter intuitive outcome of zero-shot coordination.}
\label{fig:lever}
\end{center}
\vskip -0.2in
\end{figure}

To show the intuition behind OP consider the following game: you need to coordinate with an unknown stranger by independently choosing one from a set of 10 different levers (Figure \ref{fig:lever}a). If both of you pick the same lever a reward of 1 point is paid out, otherwise you leave the game empty-handed. Clearly, without any prior coordination the only option is to pick one of the levers at random, leading to an expected reward of $1 / 10 = 0.1$. 

Next we consider a game that instead only pays $0.9$ for one of the levers, keeping all other levers unchanged (Figure \ref{fig:lever}b). How does this change the coordination problem? 

From the point of view of the MG the $1$ payoff levers have no labels and so are symmetric. Since agents cannot coordinate on how to break symmetries, picking one of the $1.0$ levers leads to $0.11$ expected return. By contrast, OP suggests the choice of the $0.9$ lever. 

We note that this example illustrates another facet of OP: it is an equilibrium in meta-strategies. That is, neither agent wishes to deviate from OP as a reasoning strategy if the other agent is using it.

Note that OP does not use any action labels. Instead OP uses only features of the problem description to coordinate. Furthermore, note that the OP policy in this setting is the only policy that would \emph{never} be chosen by the types of algorithms that try to use self-play to optimize team performance, e.g. VDN \cite{sunehag2018value} or SAD \cite{hu2019simplified}.

The main contributions of this work are: 1) we introduce OP as a way of solving the zero-shot coordination problem, 2) we show that OP is the highest payoff meta-equilibrium for the zero-shot coordination problem, 3) we show how to implement OP using deep reinforcement learning (deep RL) based methods, and 4) we evaluate OP in the cooperative card game Hanabi \cite{bard2020hanabi}.

\section{Related Work}
\label{sec:related_work}

\subsection{Self-Play in Cooperative Settings}
There is a large body of research on constructing agents to do well in positive-sum games. Self-play, if it converges, converges to an equilibrium of the game and so in purely cooperative games SP agents will be able to coordinate. Here the main problem is that SP may reach inefficient equilibria and so there is a large literature on pushing self-play toward higher payoff equilibria using various algorithmic innovations \citep{babes2008social,devlin2011empirical,devlin2016plan,peysakhovich2018prosocial}. However, the setting where agents play with the same agents they have been trained with (aka. centralized training with decentralized execution) is quite different from the zero-shot coordination one which we study. 

\subsection{Cooperation and Coordination}
A closely related problem to zero-shot coordination is ad-hoc teamwork \citep{stone2010ad,barrett2011empirical}. For example: a robot agent joining another existing group of agents to play soccer \citep{barrett2011empirical}. Ad-hoc teamwork differs from the zero-shot coordination problem in that it is typically framed as a learning problem of learning the policies/capabilities of other agents during interaction whereas the pure zero-shot coordination scenario is one where there is no time to update a fixed policy that is constructed during training. These problems are closely linked and incorporating ideas from this literature into algorithms like OP is an interesting question for future research. However, another difference is that zero-shot agents only need to coordinate well with teams of agents that are optimized for the zero-shot setting, rather than arbitrary teams self-play of agents.

There is recent work looking at the situation where the one RL agent, trained separately, must join a group of new AI agents or humans \citep{lerer2018learning,tucker2020adversarially,carroll2019utility}. These papers focus on using small amounts of observed behavior from partnered test time agents to either guide self-play to selecting the equilibrium (or ``social convention'') of the existing agents \citep{lerer2018learning,tucker2020adversarially} or allow building a human model which can be used to learn an approximate best response using RL \citep{carroll2019utility}. This setting is related, but zero-shot coordination gives no behavioral data to either agent to guide self-play or allow building a model of the other agent. Instead, zero-shot makes the assumption that test-time agents being themselves are optimized for the zero-shot setting (rather than the SP setting). 

\subsection{Game Theory and Tacit Coordination}
Within behavorial game theory a large body of work considers coordination based on ``focal points'' or other shared grounding such as the famous ``you lost your friend in New York City, where are you going to meet?'' coordination problem \citep{schelling1980strategy,mehta1994nature}. However, such focal points typically come from the fact that these coordination problems are not just abstract but grounded in \emph{exogenous} features, \emph{action labels}, that are meaningful due to a prior shared context.  The zero-shot coordination setting thus is a special form of the tacit coordination problem in which there are no shared exogenous features between the different agents and OP can be thought of as a way to coordinate in this setting.

There is also a large theoretical literature on learning and evolving coordination \citep{nowak2006evolutionary}. However, as with the self-play literature, it focuses on long run outcomes within a single group of agents learning or evolving together and does not typically focus on the question of engineering agents as we do.

\subsection{Predicting Human Decision Making}
Clearly, if we were able to accurately predict how our human counterparts are going to act in any given situation, the zero-shot coordination with human counterparts would reduce to learning a best response to those predicted actions. There is a large body of work using formal models to predict and understand human decision making \citep{camerer2011behavioral} and recent work that incorporates machine learning into this question \citep{wright2010beyond,hartford2016deep,peysakhovich2017using,kleinberg2017theory,fudenberg2019predicting}. However, the majority of this research focuses on extremely simple settings such as small normal form games \citep{wright2010beyond,hartford2016deep,fudenberg2019predicting} or single decision problems \citep{peysakhovich2017using,kleinberg2017theory} rather than complex cooperative settings with partial observability.

\subsection{Domain Randomization}
Our work is also related to the idea of domain randomization \citep{tobin2017domain}. In RL and supervised learning domain randomization tries to make the realized model invariant to some feature of the environment. For example, an object detector should be invariant to the exact camera angle from which a view of an object is captured. OP applies a similar idea: a policy should be invariant to how an agent's partner breaks symmetries in the underlying game.

\subsection{Exploiting Symmetries in Other Contexts}
In the single agent context, it is harder to plan in MDPs that have more states. The idea of abstraction is to use underlying symmetries to `compress' a large MDP into a simpler one, solve for the optimal strategy in the abstraction, and then lift the strategy to the original MDP. One set of such methods are MDP homomorphisms \cite{van2020plannable,ravindran2004approximate}. These, like OP, use underlying symmetries but their goal is different: they want to find payoff maximizing policies for single agent decision problems, while OP seeks to find robust policies for zero-shot coordination. Note that as the lever game illustrates robust policies are not necessarily the payoff maximizing ones. In addition, these methods do not solve the problem of equilibrium selection among `symmetric' policies in games, because the symmetry in the MDP just becomes a symmetry in the homomorphism.

A similar technique (compress, solve, then lift) is also used for finding Nash equilibria in large games like poker \cite{gilpin2007lossless}. In this case the abstraction treats `isomorphic' states equally and thus reduces the effective number of states in the game. Again, the goal is different - poker abstractions are trying to find Nash equilibrium strategies in the original game while OP uses symmetries to select among a set of possible equilibria.  

\section{Zero-Shot Coordination}
\label{sec:background}

In this paper we study fully cooperative Markov games. To construct this environment we start out with a Dec-POMDP~\cite{nair2003taming} with states $s_t \in \mathcal{S}$. There are $i=1, \cdots N$ agents who each choose actions, $a^i_t \in \mathcal{A}$ at each time step. 

The game is partially observable, $o^i_t \sim O(o|i, s_t)$ being each agent's stochastic observation function. At time $t$ each agent has an action-observation history $\tau^i_t= \{o^i_0, a^i_0, r^i_0, \cdots, o^i_t\}$ and selects action $a_i^t$ using stochastic policies of the form $\pi^i_\theta(a^i|\tau^i_t)$.  The transition function, $P(s'|s,\mathbf{a})$, conditions on the joint action, $\mathbf{a}$.

The game is fully cooperative, agents share the reward $r_t$ which is conditioned on the joint action and the state. Thus, the goal is to maximize the expected return $J = \mathbb{E}_\tau R(\tau)$, where $ R(\tau) =\sum_t \gamma^t r_t$ is calculated using the discount factor $\gamma$.

Most work on cooperative MARL focuses on a setting where agents are trained together, although they must execute their policies independently at least at test time~, \emph{e.g.} \cite{lowe2017multi,foerster2018counterfactual,foerster2018bayesian}. The goal is to construct \textit{learning rules}, i.e. functions that map Markov games to (joint) policies that select policies for each agent that together maximize expected discounted return. Because agents are trained together, these policies may be arbitrarily complex.

We are instead interested in achieving high returns with partners that were not trained together with our agent. Instead, we will frame the problem as follows: suppose that multiple independent AI designers will construct agents that have to interact in various but ex-ante unknown Dec-POMDPs without being able to coordinate beforehand, what learning rule should these designers agree on? To make this even more concrete, consider the case of independent autonomous vehicles made by multiple firms which have to interact in novel traffic situations on a daily basis.

\begin{wrapfigure}{r}{0.5\columnwidth}
\begin{center}
\centerline{\includegraphics[width=0.5\columnwidth]{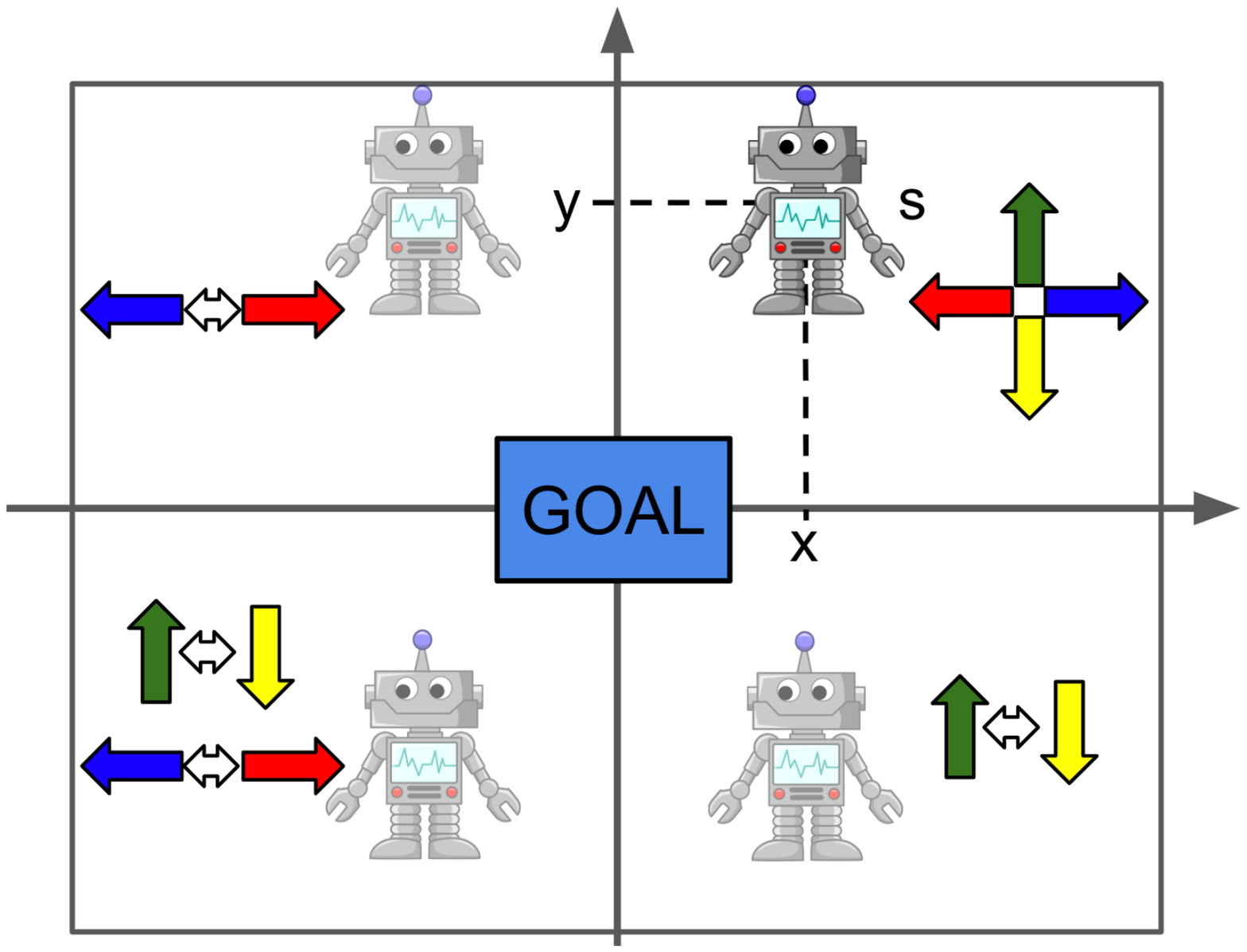}}
\caption{}
\label{fig:front-fig}
\end{center}
\vspace{-2em}
\end{wrapfigure}

The first key concept we introduce is the class of \textit{equivalence mappings}, $\Phi$, for a given Dec-POMDP.

Each element of $\Phi$ is a bijection of each of $\mathcal{S}$, $O$, and $\mathcal{A}$ onto itself, such that it leaves the Dec-POMDP unchanged:
\begin{align*}
   \phi \in \Phi \iff &  P(\phi(s')|\phi(s), \phi(a))  = P(s'|s,a) \\
    \land & R(\phi(s'),\phi(a), \phi(s) ) = R(s',a,s) \\
     \land & O(\phi(o)|\phi(s), \phi(a), i) = O(o|s, a, i) \\
     & \text{where equalities apply } \forall{s',s,a'}
\end{align*}

In other words, $\Phi$ describes the symmetries in the underlying Dec-POMDP. We note that our notation is heavily overloaded since each $\phi$ can act on actions, states and the observation function, so $\phi$ shorthand for $\phi = \{ \phi_\mathcal{S}, \phi_\mathcal{A}, \phi_O\}$. 

Next, we extend $\phi$ to also act on trajectories:
\begin{align*}
 \phi( \tau^i_t )= \{ \phi( o^i_0), \phi( a^i_0),  \phi( r_0), \cdots , \phi( o^i_t) \}.
\end{align*}

At this point an example might be helpful: consider a gridworld with a robot, shown in Figure~\ref{fig:front-fig}, that can move in the 4 cardinal directions. In our example the goal is in the middle of the room, which leaves two axis of symmetry: We can invert either the x-axis, the y-axis or both, as long as we make the corresponding changes to the action space, for example mapping ``up'' to ``down'' and vice versa when inverting the y-axis.

In a similar way, we can extend $\phi$ to act on policies $\pi$, as follows:
\begin{align*}
    \pi' = \phi(\pi) \iff \pi'(\phi(a)|\phi(\tau)) = \pi(a|\tau)\vspace{1mm},\vspace{1mm}\forall \tau, a
\end{align*}

These symmetries are the ``payoff irrelevant'' parts of the Dec-POMDP. They come from the fact that the actions and states in the Dec-POMDP do not come with labels and so taking a policy and permuting it with respect to these symmetries does not change the outcome of interest: the trajectory and the reward.

It is precisely these symmetries that can cause problems for self-play trained agents. Since agents are trained together, they can coordinate on how to break symmetries. However, there is no guarantee that multiple SP agents trained separately will break symmetries in the same way. In this case when they are paired together their policies may fail spectacularly.

The goal of OP, then, will be to build policies which are maximally robust to this failure mode.

\section{Other Play}
\label{sec:method}

We consider the $2$ agent case for ease of notation with $\pi^1, \pi^2$ denoting each agent's component of the policy and $\mathbf{\pi}$ denoting the joint policy.

First, consider self-play (SP) learning rule. This is the learning rule which tries to optimize the following objective: 
\begin{align}
    \mathbf{\pi}^* = \arg \max_{\mathbf{\pi}} J(\pi^1, \pi^2)
    \label{eq:selfplay}
\end{align}
When the Dec-POMDP is tabular we can solve this via various methods. When it is not, deep reinforcement learning (deep RL) can be used to apply function approximation and gradient based optimization of this objective function. Though there is a large literature focusing on various issues in multi-agent optimization \cite{busoniu2006multi,hernandez2019survey}, our paper is agnostic to the precise method used.

These policies can be arbitrary and in complicated Dec-POMDPs multiple maxima to Equation \ref{eq:selfplay} will often exist. These multiple policies can (and, as we will see in our experiments, often will) use coordinated symmetry breaking to receive high payoffs. Therefore, 2 matched, separately trained, SP agents will not necessarily receive the same payoff with each other as they receive with themselves.

To alleviate this issue, we need to make the optimization problem more robust to the symmetry breaking.

Let us consider the point of view of constructing a strategy for agent $1$ where agent $2$ will be the unknown novel partner. The \emph{other-play} (OP) objective function for agent $1$ maximizes expected return when randomly matched with a symmetry-equivalent policy of agent $2$ rather than with a particular one. In other words, we perform a version of self-play where agents are not assumed to be able to coordinate on exactly how to break symmetries. 

\begin{align}
    \mathbf{\pi}^* = \arg \max_{\mathbf{\pi}} \mathbb{E}_{\phi \sim \Phi} \hspace{1mm} { J(\pi^1,  \phi(\pi^2)) }
    \label{eq:otherplay}
\end{align}

Here the expectation is taken with respect to a uniform distribution on $\Phi$. We call this expected return $J_{OP}$. We will now consider what policies maximize $J_{OP}$.

\begin{lemma}
\begin{align*}
    J(\pi_A, \pi_B) = J(\phi(\pi_A^1), \phi(\pi_B^2)), \hspace{2mm} \forall \phi\in \Phi, \pi_A, \pi_B
\end{align*}
\label{lemma:permute}
\end{lemma}
\vspace{-3mm}
This Lemma follows directly from the fact that the MDP is invariant to any $\phi \in \Phi$.

\begin{lemma}
\begin{align*}
    \{\phi \cdot \phi' : \phi' \in \Phi \} = \Phi \hspace{1mm},\hspace{1mm} \forall \phi \in \Phi
\end{align*}
\label{lemma:bijection}
\end{lemma}
\vspace{-3mm}
This Lemma follows from the fact that $\Phi$ is a bijection.

\begin{prop}
The expected OP return of $\pi$ is equal to the expected return of each player independently playing a policy $\pi_\Phi^i$ which is the uniform mixture of $\phi(\pi^i)$ for all $\phi \in \Phi$.
\end{prop}

\begin{proof}
\begin{align}
    J_{OP}(\pi) & = \mathbb{E}_{\phi \sim \Phi}  \hspace{1mm} { J(\pi^1,  \phi(\pi^2)) } \\
    & = \mathbb{E}_{\phi_1 \sim \Phi, \phi_2 \sim \Phi} \hspace{1mm} { J(\phi_1(\pi^1), \phi_1(\phi_2(\pi^2))) }  \label{eq:permute1} \\
    & = \mathbb{E}_{\phi_1 \sim \Phi, \phi_2 \sim \Phi} \hspace{1mm} { J(\phi_1(\pi^1), \phi_2(\pi^2)) } \label{eq:permute2}\\
    & = J(\pi_{\Phi})
\end{align}
 (\ref{eq:permute1}) follows from Lemma \ref{lemma:permute}, (\ref{eq:permute2}) follows from Lemma \ref{lemma:bijection}.

\end{proof}

\begin{corollary}
The distribution $\pi^*_{OP}$ produced by OP will be the uniform mixture $\pi_\Phi$ with the highest return $J(\pi_\Phi)$.
\label{cor:1}
\end{corollary}

Let $\mathcal{L}_i$ be the set of learning rules which input a Dec-POMDP and output a policy for agent $i$. 

A meta-equilibrium is a learning rule for each agent such that neither agent can improve their expected payoff by unilaterally deviating to a different learning rule.

\begin{prop}
If agent $i$ uses OP as their learning rule then OP is a payoff maximizing learning rule for the agent's partner. Furthermore both agents using OP is the best possible meta-equilibrium.
\end{prop}

\begin{proof}
Since the Dec-POMDP has no labels for actions and states, $\mathcal{L}_i$ must choose all $\phi(\pi)$ with equal probability. Among these possible outputs, $\pi^*_{OP}$ maximizes the return by Corollary \ref{cor:1}.
\end{proof}

\section{Implementing Other Play via Deep RL}
We now turn to optimizing the OP objective function. In many applications of interest the Dec-POMDP is not tabular. Thus, deep RL algorithms use function approximation for the state space and attempt to find local maxima of Equation \ref{eq:selfplay} using self-play reinforcement learning. 

We show how to adapt this method to optimize the other-play objective (Equation \ref{eq:otherplay}). This amounts to applying a very specific kind of asymmetric domain randomization~\cite{tobin2017domain} during training.

During each episode of MARL training, for each agent $i$ a random permutation $\phi_i \in \Phi$ is chosen uniformly iid from $ \Phi$, and agent $i$ observes and acts on $\phi(\mathcal{S}, O, \mathcal{A}$). Importantly, the agents act in different permutations of the same environment.


This environment randomization is equivalent to other-play, because the MDP remains constant under $\phi_i$ while the effect of agent $i$'s policy on the environment is $\phi_i(\pi_i)$. The fixed points of independent optimization of $\pi$ under this learning rule will be joint policies where each $\pi_i$ is a best response to the uniform mixture of permutations of partner policies, i.e. precisely the permutation-invariant equilibria that are the solutions of other-play.

We note that OP is fundamentally compatible with any type of optimization strategy and can be applied whenever there are symmetries in the underlying MDP.

{}

\section{Experiments}
\label{sec:experiments}

We evaluate OP in two different settings. In each setting we will compare standard SP against OP. We will perform comparisons of agents trained together to agents trained separately that are placed into a zero-shot coordination test game.

\subsection{Lever Game}
We begin with the ``lever game'' mentioned in the introduction. This environment is tabular, there are only $10$ actions possible per player. Here, during training, we use simple joint action learning, compute the true gradient with respect to the current policy and update. We show training time (i.e. expected reward with itself) and test time (zero-shot) coordination performance for both SP (optimizing equation \ref{eq:selfplay}) and OP (optimizing equation \ref{eq:otherplay}). The code is available as a notebook online here and can be executed online without downloading: \url{https://bit.ly/2vYkfI7}.

Figure~\ref{fig:matrix_game_results} shows the results. As expected, OP agents coordinate on the unique option of $0.9$ points both during the training phase and at test time. As a consequence, OP agents can carry out successful zero-shot coordination when paired with other OP agents. 

In contrast, SP agents achieve higher rewards of $1.0$ points during the training phase but entirely fail to coordinate with other, independently trained, SP agents. 

\begin{figure}[h]
\vskip -0.1in
\begin{center}
\centerline{\includegraphics[width=\columnwidth]{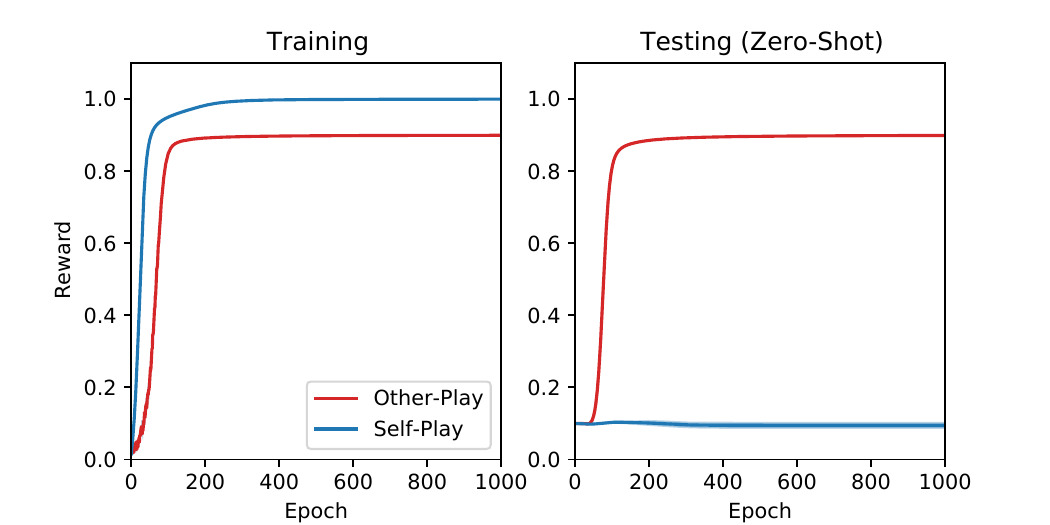}}
\caption{Train and test performance of self-play and other-play algorithms on the lever coordination game. Shown is the mean, shading is the standard error of the mean (s.e.m.), across 30 different seeds.
}
\label{fig:matrix_game_results}
\end{center}
\vskip -0.2in
\end{figure}

\subsection{Hanabi with AI Agents}
We now turn to a much more complex environment. We construct agents for the cooperative card game Hanabi, which has recently been established a benchmark environment for multi-agent decision making in partially observable settings \cite{bard2020hanabi}. 

Hanabi is a cooperative card game with the interesting twist that players cannot see their own cards and hence have to rely on receiving information from the other player (who can see their hand). In Hanabi, there are two main ways of exchanging information: first of all, players can take costly ``hint'' actions that point out subsets of cards based on rank or color. For example, hinting for ``blue'' reveals the color of all blue cards. 
Secondly, observing the actions themselves can be informative, in particular when players have pre-established conventions. The goal in Hanabi is to play cards in a legal order completing stacks of cards, one for each color. 
There are 5 color and 5 cards and the maximum points is 25. Players will lost a life token if they play a card out of order. Once they exhaust the deck or lose all 3 lives (``bomb out''), the game will terminate.

As noted, the vast majority of research on Hanabi has been in the self-play setting, in which a group of agents is trained to jointly obtain the highest possible score.
To apply OP in Hanabi we note that assuming no side information, a permutation of the colors of the cards leaves the game unchanged. We use this as our class of symmetries.

\subsection{MARL Training Details}
OP can be applied on top of any SP algorithm. In Hanabi the Simplified Action Decoder (SAD)~\cite{hu2019simplified} method achieves SOTA performance for RL agents. We use SAD as a base algorithm onto which we add OP. We use the open-sourced implementation of SAD as well as most of its hyper-parameters but with two major modifications. First, we use 2 GPUs for simulation instead of 1 as in the original paper. This doubles the data generation speed and has a profound effect on reducing the wall clock time required to achieve competitive performance. Second, we introduce extra hyper-parameters that control the network architecture to add diversity to the model capacity in order to better demonstrate the effectiveness of OP. Specifically, the network can have either 1 or 2 fully connected layers before 2 LSTM layers and can have an optional residual connection to by-pass the LSTM layers. For SP, we re-train the base SAD and the SAD + AUX variant proposed in \citet{hu2019simplified}. SAD + AUX is specifically engineered for Hanabi by adding an auxiliary task to predict whether a card is playable, discardable, or unknown. We train agents with the aforementioned 4 different network architectures. We run each hyper-parameter configuration with 3 different seeds and thus 12 models are produced for each category of \{SAD, SAD + AUX, SAD + OP, SAD + AUX + OP\}. 

\subsection{Evaluation}
We evaluate the models within the same category by pairing different models together to play the game, a process we refer to as \emph{cross-play}. Clearly, if independent training runs (``seeds'') from the same training method fail to coordinate with each other at test time it is unlikely they will coordinate with agents optimized by through a different process, let alone humans. As such, cross-play is a cheap proxy to evaluate whether a training method has potential for zero-shot coordination with human players.
Figure \ref{fig:hanabi_cross_play} shows the scores obtained between all pairs of agents. Table \ref{tab:hanabi_cross_play} shows the average within-pair and cross-play scores. We see that SAD coordinates with itself but fails to coordinate with any other SAD agent. SAD with OP, however, significantly improves the cross-play.
The effect is especially profound when the model has limited representation power. The top left corner of the graph, which corresponds to the simplest models that have only 1 fully connected layers, 2 LSTM layers and no residual connection, shows almost perfect cooperation scores. With the network growing more complicated, different strategies start to emerge and the cross-play performance drops. Auxiliary task implicitly improves cross-play scores by encouraging all agents to act basing on grounded information and confident predictions. Nonetheless, adding OP to SAD + AUX further improves performance and achieves the highest cross-play payoffs.

\begin{figure*}[h]
\begin{center}
  \includegraphics[width=0.95\textwidth]{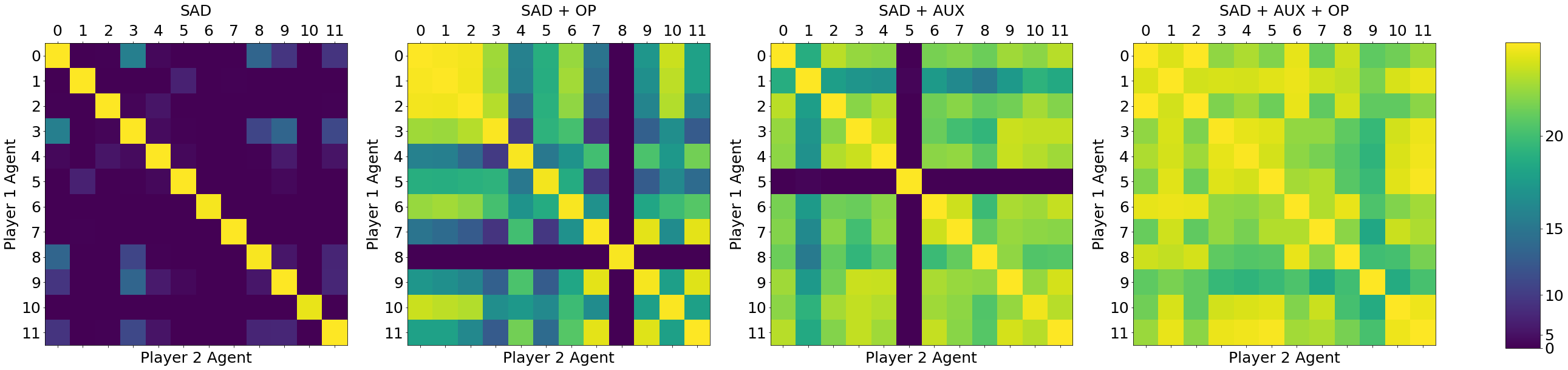}
  \caption{Cross-Play Matrix. Visualization of paired evaluation of different agents trained under the same method. y-axis represents the agent index of player 1 (first mover) and x-axis represents that of player 2. Agents 0-2: 1-layer FC, without residual connection; Agents 3-5: 1-layer FC, with residual connection; Agents 6-8: 2-layer FC, without residual connection; Agents 9-11: 2-layer FC, with residual connection. All agents have 2-layer LSTMs after the FC layers. Each block in the grid is obtained by evaluating the pair on 10K games with different seeds. Please refer to Table~\ref{tab:hanabi_cross_play} for numeric results.}
  \label{fig:hanabi_cross_play}
\end{center}
\end{figure*}

\begin{table*}[]
\begin{center}
\begin{tabular}{l c c c}
\cline{1-4}
Method         &  Cross-Play & Cross-Play(*) & Self-Play \\ \cline{1-4}
SAD            & 2.52 $\pm$ 0.34  & 3.02  $\pm$ 0.39 & 23.97 $\pm$ 0.04 \\   
SAD + OP       & 15.32 $\pm$ 0.65 & 18.28 $\pm$ 0.36 & 23.93 $\pm$ 0.02 \\   \cline{1-4}
SAD + AUX      & 17.65 $\pm$ 0.69 & 21.09 $\pm$ 0.18 & \textbf{24.09 $\pm$ 0.03} \\  
SAD + AUX + OP & \textbf{22.07 $\pm$ 0.11} & \textbf{22.49 $\pm$ 0.18} & 24.06 $\pm$ 0.02 \\   \cline{1-4}
\end{tabular}
\caption{Cross-Play Performance. The average performance of pairs of agents that are train with the same method but different network architecture and/or seeds. Please refer to Figure~\ref{fig:hanabi_cross_play} for visualization of performance for each individual pair. Cross-Play score is non-diagonal mean of each grid. Cross-Play(*) is the cross-play score after removing the worst model from the grid. Self-Play score is the score attained when agents play with the partner they are trained with.}
\label{tab:hanabi_cross_play}
\end{center}
\end{table*}

We can further study the policies resulting from these learning algorithms. Figure \ref{fig:hanabi_action_matrix} picks the agent with highest cross-play performance in each category (top row) as well as their worst possible partner (bottom row) and presents $P(a^i_t \mid a^{j}_{t-1})$ over a subset of actions averaged over time-steps in 1000 episodes generated through self-play. In other words, we ask, do the agents respond very differently to possible actions of their partner?  A large difference indicates that what an agent would do in a situation is very different from what their partner would do: a recipe for miscoordination!

We see that two paired SAD agents have very different policies and thus miscoordinate a lot. They also learn ``inhuman'' conventions that are hard for human to understand. For example, the agent hints Color5 to indicate discarding the 1st card while its partner interprets that as playing the 2nd card. OP eliminates these type of conventions. From the plot and our experience of playing with the SAD + OP agent, we find that it tends to use color hints to indicate either that the partner should save the card, or to disambiguate with a subsequent rank hint. This is not a typical strategy played by seasoned human players but is easy understand and thus makes the agent easier to cooperate with. However, due to the way we implement OP in Hanabi, it is still possible to form secretive conventions such as using all color hints to indicate a specific move. For example, the worst partner of SAD + OP uses all color hints to indicate playing the 5th card.

\begin{figure*}[h]
\begin{center}
  \includegraphics[width=0.95\textwidth]{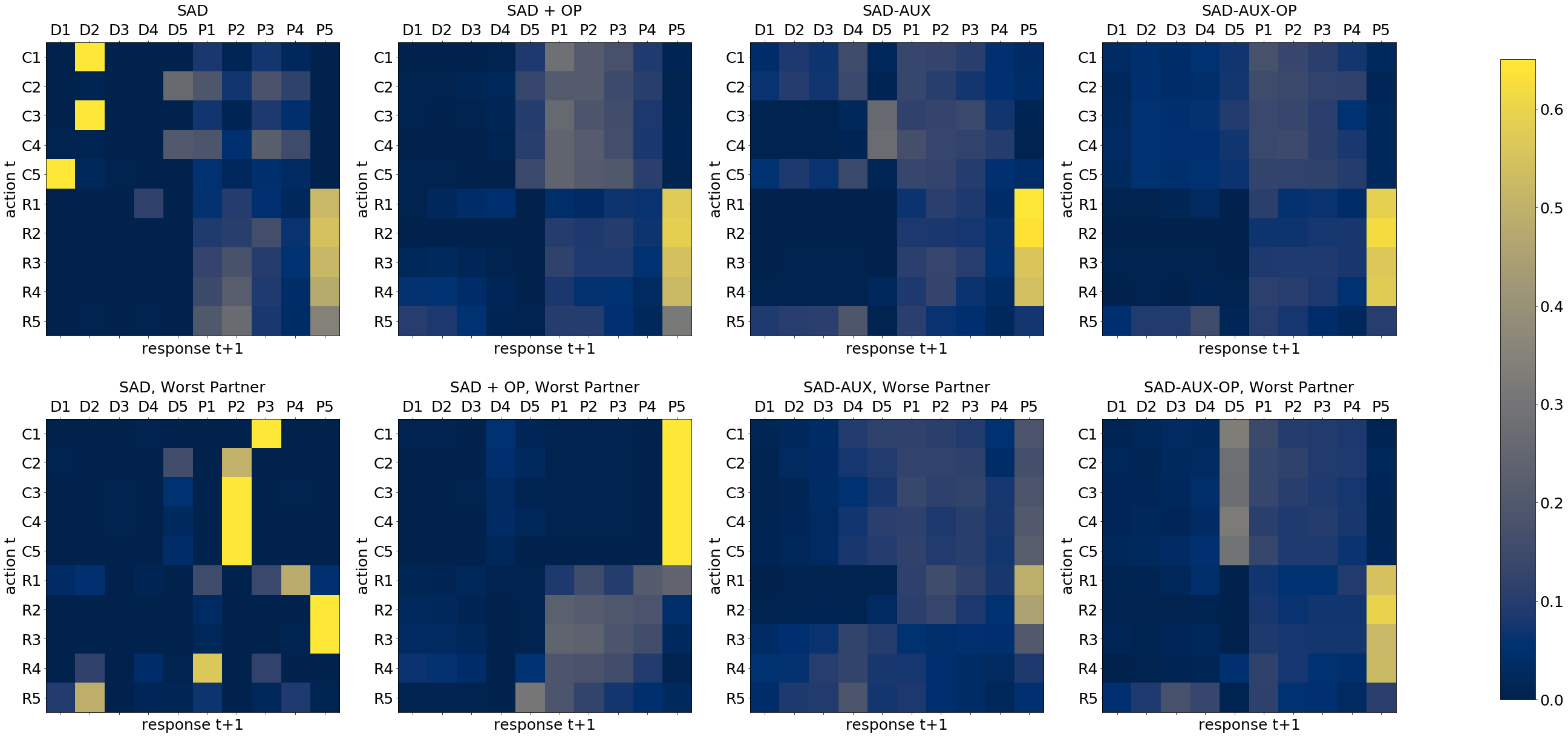}
  \caption{$P(a^i_t \mid a^{j}_{t-1})$ Matrices. Each subplot can be used as a proxy to roughly decipher conventions of an agent. The y-axis means the action taken at time-step $t$ and x-axis means the percentage of each action as response at time-step $t+1$. Here we only show a quarter of the matrix that corresponds to the interaction between color/rank hint and play/discard positions. C1-C5 and R1-R5 mean hinting 5 different colors and ranks respectively while D1-D5 and P1-P5 mean discarding and playing 1st-5th cards of the hand. For each plot, we take an agent and run 1000 episodes of self-play to compute statistics. The agents that achieved the highest cross-play scores in Figure~\ref{fig:hanabi_cross_play} are used to generate the top row and their worst partners are chosen to render the bottom row.} \label{fig:hanabi_action_matrix}
 
\end{center}
\end{figure*}

\subsection{Hanabi with Humans}
So far we have focused on AI agents that play other AI agents and have shown that OP is a meta-equilibrium with respect to learning rules in the zero-shot setting. 

We now ask: do human strategies in Hanabi also have an OP-like quality? In other words, do OP agents perform well with humans?

To begin to answer this question we recruited 20 individuals from a board game club. These individuals were familiar with Hanabi but not expert players. We asked each individual to play a game of Hanabi with two bots, in random order, using the user interface open-sourced by ~\cite{lerer2019improving}. We note that we did not provide the participants with any information about the bots, either regarding their strategy or the method through which they were trained. 

For testing we selected our best SAD + AUX + OP agent based on the cross-play performance (henceforth OP bot). We also have individuals play with the SOTA self-play agent from ~\cite{hu2019simplified} (henceforth SP bot). We download models from their GitHub repo and pick the model based on cross-play scores. For reference, the SP model used here gets 23.99 in self-play and 20.99 in cross-play with other agents where the only difference among them is seed.

Since in Hanabi the exact deck being used can make a huge difference (for example, some hands are unwinnable), to reduce the variance of our results we play each seed by two different players, one for our OP agent, and one for the control. Importantly, to prevent any adaptation advantages, we alternate the order between which bot came first across different participants. 

Humans achieved an average score of 15.75 (s.e.m. 1.23) with the OP bot and ``bombed out'' 45\% of games. Thus the OP bot, which has high cross-play scores with other OP bots is also able to play with humans. Note, in our counting convention players keep the current score when they bomb out, which we believe is more appropriate for the zero-shot setting.

By comparison, humans paired with the SP bot achieve an average score of 9.15 (s.e.m. 1.18) and an 85\% bomb rate. To the best of our knowledge, the only other research involving human-AI collaboration in Hanabi is~\cite{eger2017intentional}. Here, a hand-coded AI agent, designed to play well with humans is used and achieves an average of around 15.0 points when paired with humans. 
Beyond the average scores and ``bomb-out'' rates, we thus also have access to pairwise comparisons for the two bots when playing two different people on the same deck. 

\begin{figure}[h]
\begin{center}
\centerline{\includegraphics[width=0.7\columnwidth]{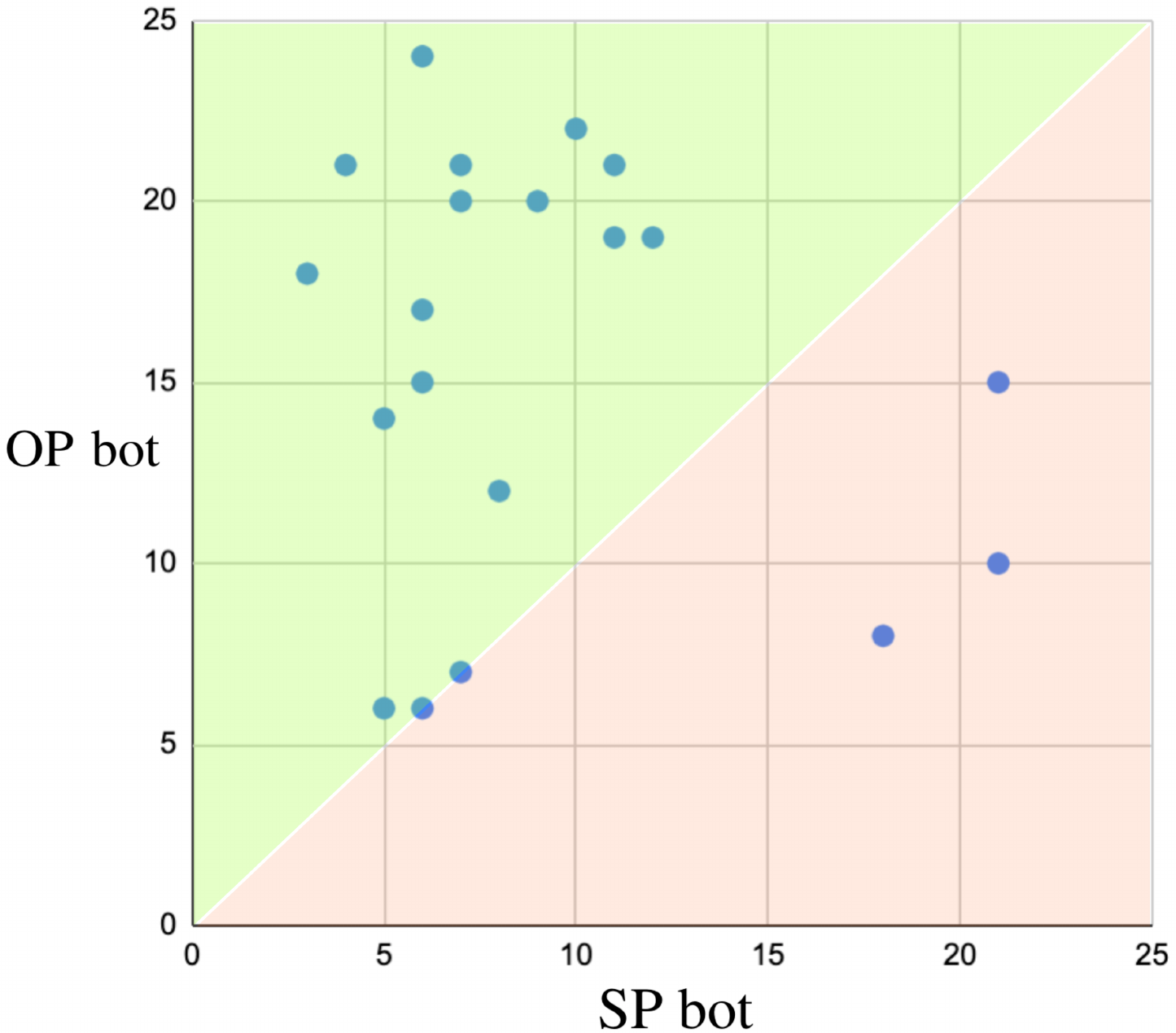}}
\caption{Shown are all scores obtained by human testing. Each blue dot is one seed, when humans are matched with the SP bot (x-axis) and the OP bot (y-axis).}
\label{fig:human_data_plot}
\end{center}
\vskip -0.2in
\end{figure}

These preliminary numbers confirm that our OP bot significantly outperformed the SOTA self-play bot from~\cite{hu2019simplified} when paired with humans. In particular, OP won 15 out of the 20 per-seed comparisons and tied in 2 cases, losing to the control group for 3 seeds (p=0.004\footnote{exact binomial test of the null hypothesis being that P(OP higher score) $\le$ P(control higher score).}).

Of course, these results do not suggest that OP will work in \textit{every} zero-shot coordination where AI agents need to cooperate with humans. However, they are encouraging and suggest that OP is a fruitful research direction for the important problem of human-AI coordination.

\section{Other Attempts}
While attempting to make progress on the zero-shot coordination problem in Hanabi we tried a variety of approaches. Here we discuss some other approaches that seemed promising but did not yield agents that were able to coordinate with agents they were not trained with. While this does not necessarily mean these approaches are doomed to failure, we report these results as information for other researchers interested in this problem.

In particular we tried multi-agent RL adaptations of cognitive hierarchies~\cite{stahl1993evolution}, k-level reasoning~\cite{costa2001cognition} and training a population of agents. 
Our original inspiration was that both cognitive hierarchies and k-level reasoning should reduce the tendency towards arbitrary symmetry breaking and have been shown to produce human like decision making in other settings~\cite{wright2010beyond}. Similarly, population based approaches are gaining popularity for regularizing communication protocols in the field of emergent communication, see e.g. ~\cite{fitzgerald2019populate,tieleman2018shaping,lowe2019learning}.
We found that none of these approaches produced high cross-play performance in Hanabi, which we now consider a necessary condition for high zero-shot performance with humans. 
In hind-sight, considering that all of these approaches would necessarily fail in the matrix game example, this is not at all surprising. Still, to help future researchers learn from our endeavours, we are adding all results to the supplementary material and will open-source the corresponding agents.

\section{Conclusion}
\label{sec:conclusion}

We have shown that a simple expansion of self-play which we call \emph{other-play} can construct agents that are better able to zero-shot coordinate with partners they have not seen before. We have proven theoretical properties of the OP strategy, shown how to implement it with deep RL, and shown in experiments with the cooperative card game Hanabi that OP can construct robust agents that can play with other AIs as well as with humans.

We do not claim that OP is a silver bullet for all zero-shot coordination problems. In particular, because OP is a modification of the SP algorithm, it can be combined with many of the algorithmic innovations that have been developed to improve SP in various games \citep{lanctot2017unified,foerster2018counterfactual,lowe2017multi,foerster2017stabilising}. Thus, we believe that this represents an exciting research direction for those interested in moving deep RL beyond two-player, zero-sum environments to ones involving coordination and cooperation. Currently we are assuming that the symmetries $\Phi$ are given to the algorithm. However, in principle, discovering the symmetries of an MDP is another optimization problem, which opens interesting avenues for future work.
\section*{Acknowledgements}
We would like to thank Noam Brown for encouraging discussions and Pratik Ringshia for help with the UI for human testing. 
We would also like to thank our human participants for offering to help. 

\bibliography{references}
\bibliographystyle{icml2019}

\newpage
\appendix
\section{Details on Other Attempts}
\label{supp:other_attempt}

\begin{table}[h]
\begin{center}
\begin{tabular}{l c c}
\cline{1-3}
k  &  Cross-Play & Self-Play \\ \cline{1-3}
1  & 1.06 $\pm$ 0.04 & 1.04 $\pm$ 0.06 \\  
2  & 0.95 $\pm$ 0.18 & 0.99 $\pm$ 0.33 \\  
3  & 1.49 $\pm$ 0.11 & 2.63 $\pm$ 0.34 \\  
4  & 2.48 $\pm$ 0.22 & 5.89 $\pm$ 0.75 \\  
5  & 2.04 $\pm$ 0.35 & 7.22 $\pm$ 0.56 \\\cline{1-3}
\end{tabular}
\caption{Cognitive Hierarchies Performance. We train CH for 5 levels with 3 seeds. The cross-play and self-play results are computed by averaging scores of intra-level pairing of agents trained with different seeds. Cross-play score is averaged over 6 pairs and self-play score is averaged over 3 pairs for each cell.}
\label{tab:ch}
\end{center}
\end{table}

In this section we provide more details on other attempts mentioned previously. 

The core idea behind cognitive hierarchies (CH)~\cite{stahl1993evolution} and k-level reasoning~\cite{costa2001cognition} is to train a sequence of $K$ agents of different capabilities. The final agents may hopefully learn strategies that cross-play well through such explicit route of evolution. In our implementation of CH, the first agent $a^{(0)}$ in the sequence is a random agent that pick actions uniformly regardless of the state. The $k$th agent $a^{(k)}$ is trained to be the ``best response'' to the pool of agents $\{a^{(0)}, ... , a^{(k-1)}\}$. Intuitively, this means that the first agent will learn to play based only on the hinted facts, i.e. ground information, to play Hanabi because $a^{(0)}$'s actions contain no intentions nor conventions. $a^{(2)}$ can then learn to give more useful hints and the subsequent agents may learn more complicated behaviors. In k-level, the $k$th agent $a^{(k)}$ only learns to best respond $a^{(k-1)}$ and other aspects remain the same. Because the agents are trained with a random agent $a^{(0)}$ and they will ``bomb out'' inevitably, we alter the reward scheme so that the agents receive reward 0, instead of the negative of current score, when they lose all life tokens.

The performance of CH is shown in Table~\ref{tab:ch}. The most prominent phenomenon is that CH converges quite slowly, due to the fact that it needs to cooperate with a pool of different yet primitive policies. For each level, we train the models until convergence and 5 levels normally take several days to complete. For reference, it roughly takes less than 20 hours for SAD and our other-play agents to reach 23 points in self-play under the same settings and hardware. The prohibitive cost in time and computation make CH unsuitable for complicated tasks like MARL in Hanabi. Moreover, even though the self-play score is low, we can already see a clear performance gap between self-play and cross-play, making it safe to assume that CH will not work well in zero-shot coordination.

\begin{table}[]
\begin{center}
\begin{tabular}{l c c}
\cline{1-3}
k  &  Cross-Play & Self-Play  \\ \cline{1-3}
1  & 0.73 $\pm$ 0.18 & 0.95 $\pm$ 0.36 \\ 
2  & 0.50 $\pm$ 0.02 & 0.54 $\pm$ 0.13 \\ 
3  & 2.99 $\pm$ 0.11 & 3.24 $\pm$ 0.24 \\ 
4  & 1.71 $\pm$ 0.12 & 2.57 $\pm$ 0.11 \\  
5  & 6.14 $\pm$ 0.58 & 7.27 $\pm$ 1.48 \\
6  & 2.08 $\pm$ 0.22 & 4.52 $\pm$ 1.05 \\
7  & 6.28 $\pm$ 0.92 & 8.82 $\pm$ 2.17 \\
8  & 1.82 $\pm$ 0.25 & 4.89 $\pm$ 1.95 \\
9  & 6.87 $\pm$ 0.91 & 10.26 $\pm$ 2.52 \\
10 & 2.05 $\pm$ 0.28 & 6.54 $\pm$ 2.65 \\

\cline{1-3}
\end{tabular}
\caption{K-Level Performance. We train K-Level with $K=10$ and 3 seeds. The cross-play and self-play scores are computed by intra-level pairing of agents trained with different seeds.}
\label{tab:k-level}
\end{center}
\end{table}

In Table~\ref{tab:k-level}, we show results of K-Level method trained for 10 levels. Despite the gap between cross-play and self-play being smaller, this method suffers from non-monotonic improvements between levels and the same high cost in time and sample complexity.

\begin{table}[h]
\begin{center}
\begin{tabular}{l c c}
\cline{1-3}
            &  Population1 & Population2  \\ \cline{1-3}
Population1 & 23.46 $\pm$ 0.01 & 19.97 $\pm$ 0.06 \\  
Population2 & 20.16 $\pm$ 0.06 & 23.44 $\pm$ 0.01 \\  
\cline{1-3}
\end{tabular}
\caption{Performance of Population Based Method. Each cell is computed by pairing all agents from one population with those from the other population and then average the scores. Diagonal can be seen as self-play score and non-diagonal can be seen as cross-play score under a population setting.}
\label{tab:pop}
\end{center}
\end{table}

\begin{figure}[h]
\begin{center}
\centerline{\includegraphics[width=\columnwidth]{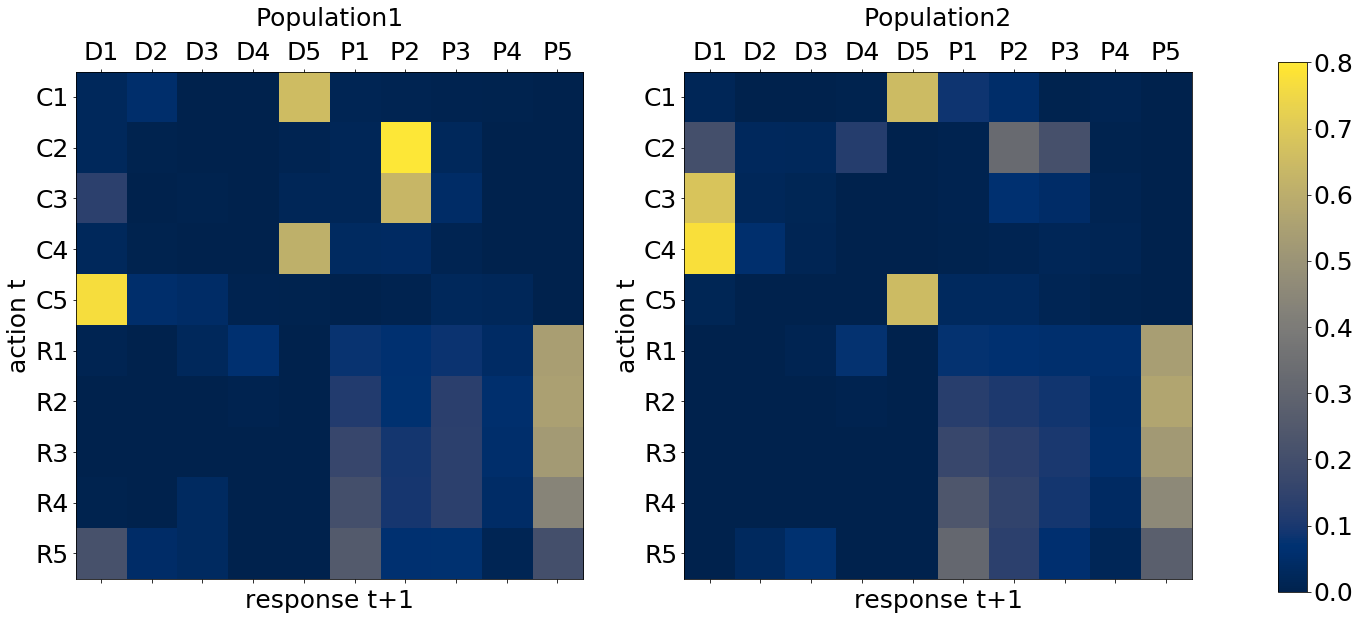}}
\caption{$P(a^i_t \mid a^{j}_{t-1})$ matrices of one model from each population. The semantic of the visualization is identical to Figure~\ref{fig:hanabi_action_matrix}.}
\label{fig:population_action_matrix}
\end{center}
\end{figure}

Different from CH and K-Level where agents are trained sequentially, population based approaches~\cite{fitzgerald2019populate, tieleman2018shaping,lowe2019learning} train agents simultaneously by pairing distinct agents together to generate samples. We briefly experimented with a simple population setting where we initialize $N$ different agents with their private replay buffers. They are uniformly paired together with each other to generate samples and write their observation action sequences into the their own buffer. Each agent is optimized independently at each training step. We train 2 populations with different seeds. Each of them contains 4 agents initialized differently. The numerical results are shown in Table~\ref{tab:pop}. This method can achieve decent cross-play scores. It is worth noting that the diversity between the hyper-parameters of the two populations is much smaller than that of the experiments shown in Figure~\ref{fig:hanabi_cross_play} so that they are not directly comparable. However, a closer look at their respective policy through $P(a^i_t \mid a^{j}_{t-1})$ matrices reveals the problem. The way they use color hints not only differs greatly from each other but also breaks the color symmetry of the game, which is the exact problem other-play tries to solve. Qualitatively they are hard for human to play with. They manage to achieve good cross-play scores because the agents seldom use color hints.

\end{document}